\documentclass[letterpaper]{article} 
\usepackage{aaai24}  
\usepackage{times}  
\usepackage{helvet}  
\usepackage{courier}  
\usepackage[hyphens]{url}  
\usepackage{graphicx} 
\usepackage{natbib}  
\usepackage{adjustbox}
\usepackage{caption} 
\usepackage{algorithm}
\usepackage{algpseudocode}
\usepackage{dsfont}
\usepackage{bm}
\usepackage{booktabs}
\usepackage{makecell} 
\renewcommand{\toprule}{\Xhline{0.1em}}
\renewcommand{\midrule}{\Xhline{0.05em}}
\renewcommand{\bottomrule}{\Xhline{0.09em}}
\usepackage{multirow}

\usepackage[table,xcdraw]{xcolor}
\usepackage{amsmath}
\usepackage{amsthm}

\newtheorem{thm}{Theorem}

\title{Selective Focus: Investigating Semantics Sensitivity \\in Post-training Quantization for Lane Detection}

\author {
    Yunqian Fan\textsuperscript{\rm 1,2},
    Xiuying Wei\textsuperscript{\rm 2},
    Ruihao Gong\textsuperscript{\rm 2,4},
    Yuqing Ma\textsuperscript{\rm 3, 4},
    Xiangguo Zhang\textsuperscript{\rm 2}
    Qi Zhang\textsuperscript{\rm 2}
    Xianglong Liu\textsuperscript{\rm 4}\thanks{Corresponding author.}
}
\affiliations {
    \textsuperscript{\rm 1}School of Information Science and Technology, ShanghaiTech University 
    \textsuperscript{\rm 2}SenseTime Research \\
    \textsuperscript{\rm 3}Institute of Artificial Intelligence, Beihang University, Beijing, China \\
    \textsuperscript{\rm 4}State Key Laboratory of Complex \& Critical Software Environment, Beihang University, Beijing, China  \\
    
    fanyq2022@shanghaitech.edu.cn, weixiuying966@gmail.com, \\ \{gongruihao, zhangxiangguo, zhangqi3\}@sensetime.com, 
    \{mayuqing, xlliu\}@buaa.edu.cn
}

\usepackage{bibentry}

\begin{document}

\maketitle

\begin{abstract}
	Lane detection (LD) plays a crucial role in enhancing the L2+ capabilities of autonomous driving, capturing widespread attention. The Post-Processing Quantization (PTQ) could facilitate the practical application of LD models, enabling fast speeds and limited memories without labeled data. However, prior PTQ methods do not consider the complex LD outputs that contain physical semantics, such as offsets, locations, etc., and thus cannot be directly applied to LD models. 
	In this paper, we pioneeringly investigate \textbf{semantic sensitivity} to post-processing for lane detection with a novel Lane Distortion Score. Moreover, we identify two main factors impacting the LD performance after quantization, namely intra-head sensitivity and inter-head sensitivity, where a small quantization error in specific semantics can cause significant lane distortion. 
	Thus, we propose a Selective Focus framework deployed with Semantic Guided Focus and Sensitivity Aware Selection modules, to incorporate post-processing information into PTQ reconstruction. Based on the observed intra-head sensitivity, Semantic Guided Focus is introduced to prioritize foreground-related semantics using a practical proxy. For inter-head sensitivity, we present Sensitivity Aware Selection, efficiently recognizing influential prediction heads and refining the optimization objectives at runtime. 
	Extensive experiments have been done on a wide variety of models including keypoint-, anchor-, curve-, and segmentation-based ones. Our method produces quantized models in minutes on a single GPU and can achieve 6.4\% F1 Score improvement on the CULane dataset.
\end{abstract}

\begin{figure*}[h]
	\centering
	\includegraphics{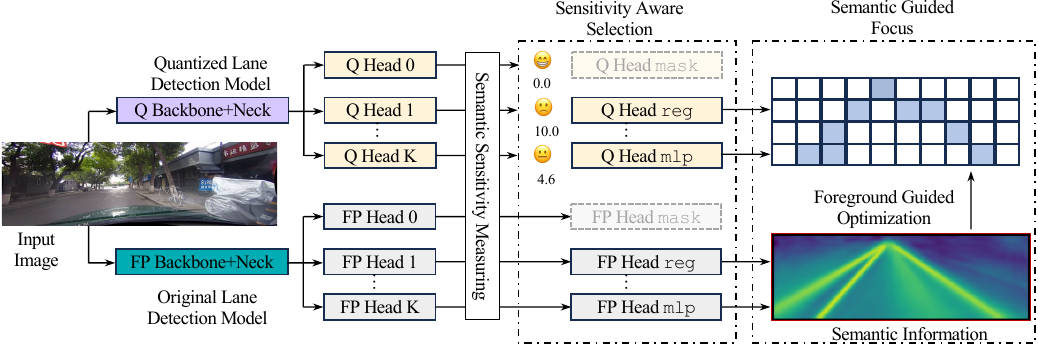}
	\caption{The framework of Selective Focus. Three modules are designed to mine the semantics sensitivity in the post-training quantized lane detection. \textbf{Semantic Sensitivity Measuring} measures the semantic sensitivity quantitatively; \textbf{Sensitivity Aware Selection} adapts the optimization objectives according to dynamic sensitivity. \textbf{Semantic Guided Focus} enables PTQ to focus on the foreground with a practical proxy.}
	\label{fig:framework}
\end{figure*}
\section{Introduction}

Deep neural networks have recently sparked great interest in autonomous driving. 
As a fundamental component in autonomous driving, lane detection~(LD) is fundamental for high-level functions such as lane departure warning, lane departure prevention, etc. Lane detection~\cite{ufld, ganet, laneatt} has garnered significant attention and undergone in-depth research, leading to substantial advancements. However, LD models are often required to run on edge devices within limited sizes, necessitating quantization and introducing formidable challenges to detection performance.

There are two prevalent quantization techniques: Quantization Aware Training (QAT) e.g. \cite{pact, lsq, lsq+, tqt}  and Post-
training Quantization (PTQ) e.g. \cite{adaquant, easyquant, qdrop, nwq}. Though QAT can often yield promising performance, it requires a longer training duration and
the whole labeled dataset, raising computation costs and safety concerns. In contrast, PTQ methods have attracted wide attention from both
industry and academia due to their speed and label-free nature.  Recently, some PTQ methods~\cite{adaround, brecq, qdrop} propose to tune the weight by reconstructing the original outputs, bringing better performance. 

LD models typically regress semantic outputs with physical meanings such as offsets, locations, and angles, and employ complex post-processing to handle these outputs. Notably, the sensitivity of these semantic outputs to post-processing varies, with certain elements having the potential to induce significant lane deformation even with minor perturbations. Prior PTQ approaches employing direct reconstruction methods on feature maps treat all outputs uniformly, overlooking post-processing information.

In this paper, we first propose the \textbf{semantic sensitivity} in lane detection models and introduce a Lane Distortion Score to measure the quantization distortion between the original LD model and the corresponding quantized counterpart. Subsequently, we investigate these sensitivities from two perspectives, namely the intra-head sensitivity and the inter-head sensitivity. Specifically, the intra-head sensitivity highlights the heightened sensitivity of a limited number of foreground (lane) regions to quantization noise during post-processing, while the inter-head sensitivity indicates the varying degrees of sensitivity to the quantization of different semantic heads over time, as shown in Figure~\ref{fig:intra}.    

To address the sensitivity problems above, we propose a Selective Focus framework to alleviate the semantic sensitivity in post-training quantization for LD models, enhancing the performance. The proposed framework is deployed with a Semantic Guided Focus module and a Sensitivity Aware Selection module, respectively targeting the intra-head sensitivity and the inter-head sensitivity. First, the Semantic Guided Focus generates practical proxies of masks from semantics, enhancing the precision of foreground lanes in post-processing.
This method guides PTQ to optimize these pivotal areas. Furthermore, the Sensitivity Aware Selection module refines optimization objectives by querying efficiently the real-time sensitivity of each head through our Lane Distortion Score across heads.
The proposed framework could tune the models efficiently by introducing the semantic information of the post-process to the optimization implicitly.

To the best of our knowledge, our work is the first to identify the role of semantic sensitivity in PTQ for lane detection models, and we hope it could offer new insight to the community. 
Extensive experiments on the widely-used CULane dataset and various leading methods validate the effectiveness and efficiency of the proposed Selective Focus framework. 
In summary, our contributions are listed as follows:
\begin{itemize}
	\item We introduce the concept of semantics sensitivity in post-processing quantization for lane detection, proposing the Lane Distortion Score metric. Our Selective Focus framework, composed of the Semantic Guided Focus and Sensitivity Aware Selection modules, addresses both intra-head and inter-head sensitivities.
	\item  Considering the intra-head semantics sensitivity, the Semantic Guided Focus module generates practical proxy masks from semantics and thus guides PTQ to optimize these pivotal areas.
	\item  Handling inter-head semantics sensitivity, the Sensitivity Aware Selection module efficiently adjusts optimization objectives based on each head's real-time sensitivity measured by our Lane Distortion Score.
	\item  Our empirical tests across datasets, models, and quantization setups endorse our approach's efficacy. Notably, under the 4-bit setup, performance gains exceed up to 6.4\%, on benchmark models with a 6x acceleration. 
\end{itemize}

\section{Related Works}
                
\subsection{Lane Detection Models}

The LD task aims to produce lane representations in the given images. Despite the different methods, they all try to set the foreground (lanes) apart from the background. 
Nowadays the models generally use Convolutional Neural Networks~(CNNs) to extract lane features, which can be divided into keypoint-based, anchor-based, segmentation-based, and parameterized-curve-based methods.

\textbf{Keypoint-based Methods} predict the mask of lane points and regress them to the real location on the corresponding lanes.  CondLaneNet~\cite{condlane} regresses offset between adjacent keypoints, while GANet~\cite{ganet} regresses offset between each keypoint to the start point of its lane. \textbf{Anchor-based Methods} model lanes as pre-defined pairs of start point and angle, and then regress lanes among them. LaneATT~\cite{laneatt} proposes an anchor attention module to aggregate global information for the regression. CLRNet~\cite{clrnet} refines the proposals with features at different scales.
\textbf{Segmentation-based Methods} predict the mask of all the lanes on the image and then cluster them into different lanes. SCNN~\cite{scnn} adopts slice-by-slice convolution modules to aggregate surrounding spatial information. RESA~\cite{resa} further extends the mechanism to aggregate global spatial to every pixel.
\textbf{Curve-based Methods} model lanes as  singular curves,  rather than sets of discrete points. For instance, LSTR~\cite{lstr} predicts the parameters for cubic curves and BézierLaneNet~\cite{beznet} predicts for Bézier curves. 

\subsection{Post-traning Quantization}

Quantization is widely used in deep learning model deployment to substantially cut down memory and computation requirements during inference, which is required by the LD models. Compared to QAT~\cite{qat, dsq, tqt, lsq,lsq+} which requires large GPU effort and the whole dataset, PTQ methods have sparked great popularity these days due to their speed and label-free property.

Common PTQ methods like OMSE~\cite{omse} and ACIQ~\cite{aciq} often identify quantization parameters to minimize the quantized error for tensors, requiring a few batches of forward passes. More recently, some methods have evolved to slightly tune the weights and reconstruct the original outputs. \cite{rfm} improves the accuracy by setting a well-defined target for the face recognition task.  AdaRound~\cite{adaround} initially proposes that adjusting the weight within a small space can be beneficial, and their layer-wise output reconstruction can yield more favorable results with only a marginal increase in optimization time. Building on this, BRECQ~\cite{brecq} suggests that the outputs of each layer still exhibit some disparity from the final outputs and thus proposes adopting a block-wise reconstruction scheme. Later, QDrop~\cite{qdrop} investigates the activation quantization under this setting and introduces random activation quantization dropping during tuning, which benefits the performance.

We also opt for model tuning through reconstruction. Nevertheless, prior techniques haven't been applied to lane detection models with multiple heads and complex post-processing functions. We discover that directly reconstructing feature maps for these models overlooks the important post-processing information, ultimately leading to sub-optimal solutions.

\begin{figure}
	\centering
	\includegraphics{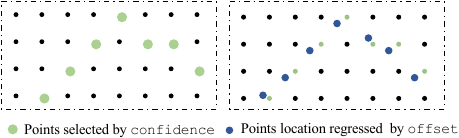}
	\caption{Example of confidence and semantics in post-process. (a) The confidence is used to predict whether a keypoint is at a certain location. (b) The offset is regressed to the shift between the real keypoint and its downscaled grid. Offset is only valid in the indices with positive confidence prediction.}
	\label{fig:div}
\end{figure}

\section{Preliminaries}\label{sec:prel}

\subsection{Notation}
In the context of lane detection models, it's important to note that each head $i$ encompasses two distinct functions: $\mathcal{S}^i(\cdot)$ and $\mathcal{C}^i(\cdot)$. The former function yields outputs with physical significance, which we refer to as \textbf{semantics}. 
These semantics are linked to physical attributes such as distance and angle. 
The latter function, $\mathcal{C}^i(\cdot)$, produces confidence outputs for each head's semantics. An illustration of how the post-process deals with confidence and semantics is shown in Figure~\ref{fig:div}.

Moreover, let $\mathbf{x}$ represent the vectors of unlabeled data from the calibration dataset $\mathcal{D}$. We use the symbol $\odot$ to indicate element-wise multiplication. Finally, $\hat{\mathcal{S}}$ generates new semantics derived from quantized models. 

\subsection{Problem Definition}

Tuning-based PTQ, as mentioned in the last section, focuses on minimizing the task quantization loss as opposed to minimizing local distance, given by:
\begin{equation}\label{eq:prevloss}\small
	\min_{\mathbf{w}} \mathds{E}_{\mathbf{x}\sim \mathcal{D}} \left(\sum_i \left\|{\hat{\mathcal{S}}}^i(\mathbf{x}) - \mathcal{S}^i(\mathbf{x}) \right\|^2_{F} + \left\|{\hat{\mathcal{C}}}^i(\mathbf{x}) - \mathcal{C}^i(\mathbf{x}) \right\|^2_{F} \right), 
\end{equation}  
where the first term corresponds to the reconstruction of the fully pixel-wise semantic outputs, while the second term pertains to the reconstruction of the confidence values. Compression methods \cite{adaround, brecq} optimizes the above equation towards layer-wise and block-wise approximation.

However, such an optimization is not suitable in LD models due to their intricate post-processing steps and semantics rooted in physical interpretation.  In the next section, we identify the importance of semantic sensitivity if the post-process. Thus, neglecting the valuable insights offered by post-processing in the optimization objective would ultimately yield less favorable outcomes.

\begin{figure}
	\centering
	\includegraphics{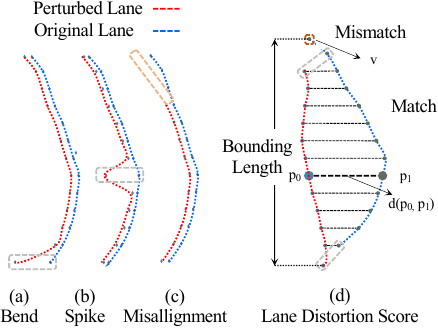}
	\caption{Common types of lane distortion caused by slight perturbation and the proposed Lane Distortion Score. (a) Bend: unexpected shifts at the terminals; (b) Spike: shifts in the middle; (c) Misalignment: missing or extra lane points. Any kind of lane distortion can be represented as a combination of those three types of distortion. (d) We measure the distortion between lanes with two types of point relationships: match and mismatch.}
	\label{fig:common}
\end{figure}

\section{Method}

In this section, we reveal an important factor that significantly impacts PTQ performance in lane detection models: the semantic sensitivity to post-processing, which has been overlooked by other PTQ research before. Then, comprehensive investigations are conducted from both intra-head and inter-head aspects. Building on sensitivity observations within and across heads, we propose a Selective Focus framework including  two novel modules, Semantic Guided Focus, and Sensitivity Aware Selection, to allocate appropriate attention to different semantics. Our framework implicitly introduces the post-processing information into quantization optimization. The pipeline is depicted in Figure~\ref{fig:framework}.

\subsection{Semantic Sensitivity}\label{sec:ssm}

\subsubsection{Sensitivity to post-process} 

Owing to the complexity of optimizing the post-process, prevailing PTQ approaches are confined to adjusting model parameters to align quantized head outputs with their full-precision counterparts (Equation~\ref{eq:prevloss}) without considering the post-process. Nevertheless, we find the importance of the post-process procedure in ensuring accurate lane generation. Disregarding post-processing information during the quantization optimization phase can result in pronounced distortions, including abrupt bends, spikes, and misalignments, even with a marginal value of (Equation~\ref{eq:prevloss}), as illustrated in Figure~\ref{fig:common}.

Motivated by this, we propose to investigate semantic sensitivity, where some outputs of heads can be so important for later post-process that small quantization errors of them can cause severe lane distortion. Incorporating information from the post-process into our model optimization would pave the way for more effective semantic reconstruction.

\subsubsection{Lane Distortion Score}

To study the sensitivity, a quantitative evaluation of lane distortion becomes imperative. Given the frequently localized deviations in lanes induced by quantization (as depicted in Figure~\ref{fig:common}), we abstain from employing the conventional Intersection over Union (IoU) metric~\cite{scnn,beznet} which overlooks these local distortions. In response, a direct but effective metric is introduced, which measures the shifts of all points from the perturbed lane to the original one, as shown in Figure~\ref{fig:common}. The score of matched points is their distance, and the score of mismatched points is a fixed penalty score $v$. Concretely, our devised metric first matches points between perturbed and normal lanes (set $\mathds{M}$), then calculates the distance for matched points ($d(p_0, p_1)$) and the penalty from mismatched ones:
\begin{equation}\label{eq:score}
	score =\sum_{(p_0, p_1) \in \mathds{M}} \frac{d(p_0, p_1)}{b(\mathds{M})} + n v,
\end{equation} 
where the bounding length normalization ($b(\mathds{M})$) is applied to accommodate lanes of varying lengths and $n$ is the number of mismatched points. More explanation of this score is available in Appendix~B. 

Equipped with the Lane Distortion Score, we can compute the distortion of lanes under perturbation, which strongly supports quantitively analysis of semantic sensitivity to post-process and further method design.

\begin{figure}[hbt]
	\centering
	\includegraphics{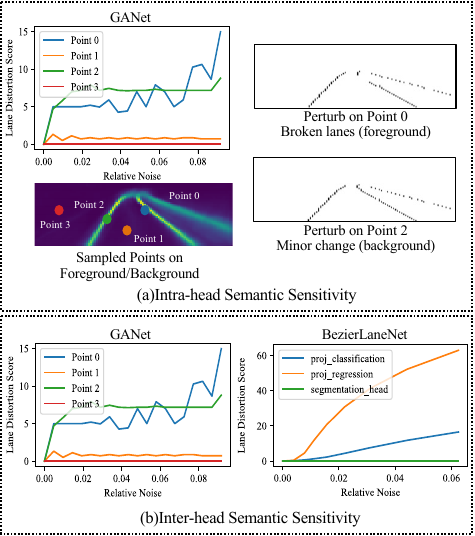}
	\caption{Illustration of semantic sensitivity. The \textbf{intra-head semantic sensitivity} shows that the foreground is sensitive to perturbation while the background is not. The \textbf{inter-head semantic sensitivity} indicates that heads are of different importance for post-processing. More illustrations are listed in the Appendix~A.}
	\label{fig:intra}
\end{figure}

\subsection{Semantic Guided Focus}
This section explores intra-head semantic sensitivity, where we observe that semantics within each head associated with the foreground region in post-processing play a more significant role and propose a method called Semantic Guided Focus. By leveraging confidence outputs, we can potentially discern between semantics linked to the foreground and background, which  enables us to prioritize the former and obtain an improved optimization in a PTQ setting without labels.
\subsubsection{Intra-head Semantic Sensitivity}

Considering the relationship between semantics (head outputs) and the post-process, some semantics pertain to the foreground region in the post-process, while others correspond to the background.  Also, it's the distortion of the foreground region (lane) that matters. Consequently, we argue that different pixels within each head exhibit distinct sensitivities. To verify, we leverage the Lane Distortion Score of each pixel by adding the same magnitude noise to them. The results showcased in Figure~\ref{fig:intra} conclusively demonstrate that injecting noise into entries associated with the foreground region can result in more severe lane deformations. Given these findings, allocating equal attention is not reasonable during optimization. 

Furthermore, it's worth noting that the number of entries about the foreground is significantly fewer than those related to the background region, which further distracts previous techniques to focus on crucial positions. Therefore, we are motivated to enhance the accurate expression of pixels tied to the foreground region and suppress that of the background.

\subsubsection{Intra-head Sensitivity Focus} 
Motivated by the findings, the core idea is to distinguish whether each element within each head will be used in the foreground region (the lane) or background. 

(1) \textbf{Reconstruction on semantic:} We first focus on the semantic term in Equation~\ref{eq:prevloss} 
and introduce a masking function $\mathcal{M}$ to achieve the distingishment. $\mathcal{M}$ removes the error term associated with the background and retains the elements for the foreground region. Then, we incorporate them with an element-wise product, and the optimization objective on semantics becomes:
 
\begin{equation}
	\label{eq:maskloss}
	\begin{aligned}
		  & \mathds{E}_{\mathbf{x}\sim \mathcal{D}}  \sum_{i}  \left\|\mathcal{M}^i(\mathbf{x})\cdot(\hat{\mathcal{S}}^i(\mathbf{x}) - \mathcal{S}^i(\mathbf{x}))\right\|^2_F. \\
	\end{aligned}
\end{equation}
However, due to the absence of lane annotations under the PTQ setting, it is unrealistic to identify exact elements tied to the foreground. Fortunately, we find an upper bound of Equation~\ref{eq:maskloss} of the models' confidence output. Here we give the theoretical finding of this upper bound:
\begin{thm}
	\label{thm:1}
	Given $\mathcal{M}$ representing a matrix function that discerns elements associated with foreground or background regions, and $\mathcal{C}$ denoting the confidence function of FP models, offering confidence scores for semantics linked with the foreground, the following inequation stands:
		
	\begin{equation}
		\label{eq:avg_regret}
		\begin{aligned}
			       & \mathds{E}_{\mathbf{x}\sim \mathcal{D}}  \sum_{i} \left\|\mathcal{M}^i(\mathbf{x})\cdot(\hat{\mathcal{S}}^i(\mathbf{x}) - \mathcal{S}^i(\mathbf{x}))\right\|^2_F \\
			\leq ~ & \mathds{E}_{\mathbf{x}\sim \mathcal{D}}  \sum_{i}\left\|\mathcal{C}^i(\mathbf{x}) \odot(\hat{\mathcal{S}}_i(\mathbf{x}) - \mathcal{S}^i(\mathbf{x}))\right\|^2_F 
		\end{aligned}
	\end{equation}
\end{thm}
Detailed proof can be found in Appendix~B. With this theorem, the semantic-related optimization target can be transformed to:
\begin{equation}
	\min_{\mathbf{w}}\mathds{E}_{\mathbf{x}\sim \mathcal{D}} \sum_{i} \left\|\mathcal{C}^i(\mathbf{x}) \odot(\hat{\mathcal{S}}_i(\mathbf{x}) - \mathcal{S}^i(\mathbf{x}))\right\|^2_F.
\end{equation} The theorm means that we can leverage the model output $\mathcal{C}(\mathbf{x})$ as a practical proxy of the annotation mask. In intuition, the knowledge from the well-trained models can be represented in its output, thus the model could generate a mask similar to the annotation. We further assume the mask get from the model follows a binomial distribution paramterized by the confidence output, then the total expectation turns to the result in the theorem. With the upper bound from it, optimization can be more tractable.

(2) \textbf{Reconstruction on confidence:} Last, we incorporate the reconstruction loss on confidence values on each head. To prevent elements associated with the background region turns into the foreground, especially under large quantization noise, we propose to enhance the alignment of confidence outputs by adopting a new parameter $\lambda$ with $\lambda > 1$. This heightened penalty on confidence outputs helps retain the fidelity of background-related pixels while enabling a concentrated focus on those linked to the foreground domain via the first term.

\begin{equation}
	\begin{aligned}\small
		\min_{\mathbf{w}}\mathds{E}_{\mathbf{x}\sim \mathcal{D}}  \sum_{i} & (\left\|\mathcal{C}^i(\mathbf{x}) \odot(\hat{\mathcal{S}}_i(\mathbf{x}) - \mathcal{S}^i(\mathbf{x}))\right\|^2_F + \\
		                                                                   & \lambda \left\| \hat{\mathcal{C}}^i(\mathbf{x}) - \mathcal{C}^i(\mathbf{x})\right\|^2_F).                          
	\end{aligned}
\end{equation}

The improved Semantic Guided Focus objective indicates a simple yet elegant principle:  a well-trained model inherently possesses the capacity to instruct itself. By utilizing model outputs for both mask estimation and background suppression in PTQ, the optimization process becomes foreground-oriented, leading to more efficient semantic alignment.

\subsection{Sensitivity Aware Selection}
We also delve into inter-head sensitivity, where we observe that specific heads are more sensitive to post-processing. This insight leads us to introduce the Sensitivity Aware Selection method, which dynamically and efficiently selects the most influential heads during PTQ reconstruction.
\subsubsection{Inter-head Semantic Sensitivity}

Considering the diverse roles played by distinct heads in post-processing, we further investigate semantic sensitivity across multiple heads. By injecting noise into each head and calculating our distortion score, Figure~\ref{fig:intra} is obtained and more results are listed in Appendix~A. It can be clearly seen that under the same perturbation magnitude, certain heads like \texttt{proj regression} in BezierLane, exhibit notably higher Lane Distortion scores and of course correspond to severely distorted lanes, compared to others. This discrepancy highlights the considerable variation in sensitivity across different heads, which encourages us to discriminate heads and thus implicitly incorporate post-process information during optimization.

\begin{algorithm}
	\caption{Sensitivity Aware Selection}\label{alg:select}
	\begin{algorithmic}
		\Require Semantic head set $\mathds{H}$, hyper parameter $k$, and calibration dataset $\mathcal{D}$. Semantic function $\mathcal{S}(\cdot)^i$ and confidence functions $\mathcal{C}(\cdot)^i$ on head $i$ on the FP model. Quantized version $\hat{\mathcal{S}}(\cdot)^i$ and $\hat{\mathcal{C}}(\cdot)^i$.
		\For{$\mathbf{x} \in \mathcal{D}$}
		\For{$i \in \mathds{H}$}
		\State Calculate full precision lanes $L$ from $\{\mathcal{S}(\mathbf{x})^i\}$
		\State Perturb the model by replacing $\mathcal{S}^i(\mathbf{x})$ with $\hat{\mathcal{S}}_i(\mathbf{x})$.
		\State Calculate noised lanes $\hat{L}$ from $\hat{\mathcal{S}}(\mathbf{x})^i$.
		\State Compute the $s$ of $(L, L')$ with Equtaion~\ref{eq:score}
		\State Update score of $i$-th head $Score_i \gets Score_i + s$ 
		\EndFor
		\EndFor
		\State Sort semantic heads by $Score$.
		\State\Return Top-$k$ of semantic head set $\mathds{H}$
	\end{algorithmic}
\end{algorithm}

\begin{table*}[hbt]
	\centering
	\begin{adjustbox}{max width=0.99\textwidth}
		\begin{tabular}{l|l|cccc|cccc|cc|cccc}
			\toprule
			Bits & Method & \multicolumn{4}{c|}{Keypoint-based} & \multicolumn{4}{c|}{Curve-based} & \multicolumn{2}{c|}{Achor-based} & \multicolumn{4}{c}{Segmentation-based} \\
			\midrule
			\multirow{3}{*}{\begin{tabular}[c]{@{}c@{}}Full \\ Precision  \end{tabular}}& \multirow{3}{*}{\begin{tabular}[c]{@{}c@{}}Model \\ Baseline  \end{tabular}}  & \multicolumn{2}{c}{CondLaneNet} & \multicolumn{2}{c|}{GANet} & \multicolumn{2}{c}{LSTR} & \multicolumn{2}{c|}{BézierLaneNet} & \multicolumn{2}{c|}{LaneATT} & \multicolumn{2}{c}{SCNN} & \multicolumn{2}{c}{RESA} \\
			\cline{3-16} 
			                      &          & Small          & Mid            & Small          & Mid            & Small          & Mid            & Small          & Mid            & Small          & Mid            & Small          & Mid            & Small          & Mid            \\
			\cline{3-16} 
			                      &          & 78.14          & 78.74          & 78.79          & 79.39          & 68.78          & 72.47          & 73.66          & 75.57          & 74.45          & 75.04          & 72.19          & 72.70          & 72.90          & 73.66          \\ 
			\midrule
			\multirow{3}{*}{W8A8} & ACIQ     & 77.95          & 78.58          & 78.58          & 79.21          & 67.50          & 72.14          & 73.43          & 75.36          & 74.34          & 74.56          & 72.03          & 72.55          & 72.76          & \textbf{73.64} \\
			                      & QDrop    & 78.04          & 78.77          & \textbf{78.70} & \textbf{79.33} & 68.40          & 72.38          & 73.63          & 75.49          & 74.33          & 74.88          & 72.04          & 72.55          & \textbf{72.80} & 73.61          \\
			                      & Ours     & \textbf{78.10} & \textbf{78.90} & 78.53          & 79.30          & \textbf{68.58} & \textbf{72.40} & \textbf{73.63} & \textbf{75.50} & \textbf{74.38} & \textbf{75.01} & \textbf{72.33} & \textbf{72.69} & 72.53          & 73.49          \\ 
			\midrule
			\multirow{6}{*}{W8A4} & ACIQ     & 58.63          & 37.67          & 5.38           & 20.18          & 47.63          & 12.49          & 23.79          & 4.77           & 54.20          & 0.64           & 62.63          & 49.26          & 54.16          & 49.62          \\
			                      & OMSE     & 69.74          & 64.29          & 69.52          & 54.10          & 55.51          & 58.12          & 62.04          & 60.57          & 64.54          & 0.90           & 65.35          & 60.53          & 66.59          & 65.93          \\
			                      & AdaRound & 67.11          & 63.64          & 39.97          & 18.14          & 51.41          & 54.11          & 56.66          & 58.55          & 64.01          & 0.96           & 65.87          & 63.65          & 59.17          & 62.78          \\
			                      & BRECQ    & 73.61          & 74.06          & 74.37          & 75.04          & 57.11          & 63.32          & 62.02          & 65.18          & 66.47          & 0.04           & 66.05          & 63.67          & 66.40          & 65.47          \\
			                      & QDrop    & 74.76          & 75.49          & 75.77          & 75.56          & 60.34          & 65.25          & 64.48          & 66.91          & 66.58          & 0.06           & 66.85          & 64.83          & 67.27          & 67.54          \\
			                      & Ours     & \textbf{75.56} & \textbf{75.74} & \textbf{76.32} & \textbf{76.51} & \textbf{63.14} & \textbf{68.15} & \textbf{68.98} & \textbf{70.01} & \textbf{69.85} & \textbf{34.53} & \textbf{69.61} & \textbf{69.51} & \textbf{69.46} & \textbf{70.60} \\
			\midrule
			\multirow{6}{*}{W4A4} & ACIQ     & 53.96          & 20.84          & 1.54           & 9.02           & 1.47           & 2.37           & 14.10          & 8.82           & 50.65          & 0.32           & 34.47          & 23.70          & 35.56          & 15.45          \\
			                      & OMSE     & 63.64          & 55.06          & 49.96          & 37.21          & 1.79           & 17.02          & 52.38          & 46.03          & 62.34          & 0.39           & 51.49          & 49.67          & 57.07          & 50.83          \\
			                      & AdaRound & 20.35          & /              & /              & /              & 20.69          & 7.95           & 50.70          & 48.60          & 34.53          & 0.00           & 6.56           & 0.03           & 68.36          & 64.57          \\
			                      & BRECQ    & 74.10          & 75.80          & 75.67          & 75.89          & 30.83          & 50.09          & 66.96          & 70.30          & 68.69          & 28.34          & 54.34          & 52.30          & 67.70          & 69.48          \\
			                      & QDrop    & 74.41          & \textbf{76.29} & \textbf{76.76} & \textbf{76.50} & 23.95          & 53.87          & 67.65          & \textbf{71.16} & 68.97          & 0.64           & 61.70          & 64.57          & 67.59          & 69.98          \\
			                      & Ours     & \textbf{74.68} & 75.48          & 76.31          & 76.26          & \textbf{34.65} & \textbf{60.56} & \textbf{68.37} & 70.00          & \textbf{69.19} & \textbf{37.59} & \textbf{68.16} & \textbf{68.27} & \textbf{69.31} & \textbf{70.01} \\
			\bottomrule
		\end{tabular}
	\end{adjustbox}
	\caption{F1-score performance comparison among different quantization algorithms and models. W8A8 means the weight and activation are all quantized into 8 bits, and so does W4A4 and W8A4. Our method achieves \textbf{superior performance} under most settings.
		}\label{tab:main}
\end{table*}

\subsubsection{Inter-head Sensitivity Selection}

To handle varied semantic sensitivity among heads, we introduce the Sensitivity Aware Selection technique, which efficiently and adaptively selects those sensitive heads during optimization. The algorithm is formulated below and its procedure can be found in the Algorithm~\ref{alg:select}. 

(1) Head selection: The noticeable differences in sensitivity prompt us to focus on the more sensitive heads, optimizing them more effectively. Naturally, given the estimated quantization noise level of each head, we can compute their own Lane Distortion Score. By ranking these scores and then selecting the top-$k$ sensitive heads, a new reconstruction loss is constructed. This approach ensures that the optimization process is cognizant of the semantic sensitivity to post-processing, leading to a better-tuned model.

(2) Adaptive selection: Moreover, the optimization process is aimed at minimizing the discrepancy between original and quantized semantics, where we recognize that the quantization loss for individual heads can evolve, leading to dynamic changes in their quantization noise levels and thus sensitivity ranking. Consequently, our selection of top-$k$ sensitive heads must adapt accordingly. In practical terms, we can reassess the Lane Distortion Scores and repeat the aforementioned step at fixed intervals of iterations.  However, this repetition could impose a considerable time overhead during tuning. 

(3) Efficient adaptive selection: To accelerate it, we propose to apply a pre-processing technique, which first employs the Monte-Carlo method to sample diverse noise levels and derive corresponding Lane Detection scores for each head, then interpolates sampled points from a continuous noise-score curve. This approach empowers us to gauge the semantic sensitivity of different heads based on their respective curves using their current quantization loss as queries, which introduces negligible computational burden during the optimization process. The details of building the noise-score curve are listed in Appendix~B.

Based on varied semantic sensitivities across heads, we adeptly and dynamically select the most sensitive heads. This implicit inclusion of post-processing guidance leads to a better-optimized model.

\section{Experiments}
Extensive experiments are conducted to prove the effectiveness of the Selective Focus framework. We first present the experiment setup, and then compare the proposed method with other state-of-the-art PTQ works and the method shows up to 6.4\% F1 score improvement. After that, the ablation study of the Selective Focus framework demonstrates the contribution of each component. Finally, we compare the efficiency of the framework with existing PTQ and QAT methods.

\subsection{Experiments Setup}
We describe the datasets and evaluation protocols, the comparison methods, and the implementation details.

\subsubsection{Datasets and Evaluation} 
We conduct comprehensive experiments on the CULane dataset and adopt its official evaluation method. CULane contains 88,880 training images and 34,680 test images from multiple scenarios, and the evaluation method provides precision, recall, and F1 score for each scenario. For brevity, we list the F1 score of the whole dataset in the main document and the left in Appendix~C. 
For models, we evaluate LD models in the four major classes: keypoint-, anchor-, segmentation-, and curve-based models, including CondLaneNet~\cite{condlane}, GANet~\cite{ganet}, LSTR~\cite{lstr}, BézierLaneNet~\cite{beznet},.LaneATT~\cite{laneatt}, SCNN~\cite{scnn}, and RESA~\cite{resa}.

\subsubsection{Implementation Details}
We implement our method based on the PyTorch framework. Weights and activations are both quantized with concrete bits denoted as W/A. Our method is calibrated with 512 unlabeled images on three kinds of quantization bits: W8A8, W8A4, and W4A4. During the optimization, we choose the Adam optimizer with a learning rate set as 0.000025 and adjust weights for 5000 iterations. Because of more computation overhead for layer-wise and block-wise reconstruction, the net-wise reconstruction is adopted here and wins a 6X speedup. 
Other hyper-parameters including  $k$ for Top-$k$ in Sensitivity Aware Selection is kept as 1 for models with two heads and 2 for others, based on our ablation studies. For more detailed implementation, please refer to Appendix~C.

\subsubsection{Comparison Methods}
We implement popular baselines including OMSE~\cite{omse}, ACIQ~\cite{aciq}, AdaRound~\cite{adaround}, BRECQ~\cite{brecq}, and  QDrop~\cite{qdrop}. AdaRound and BRECQ are implemented by leveraging a technical advancement introduced in QDrop, achieving better results for them.

\subsection{Main Results}
We conduct experiments on CULane and TuSimple~\cite{Tusimple} datasets. Table~\ref{tab:main} here shows the results on CULane, and results on TuSimple are put in Appendix~C due to space limit.

With the decreasing activation precision, the proposed method shows advanced performance consistently. For example, the method can achieve more than 3\% up to 6.4\% F1 score gain under 4-bit activation.  As the noise of the semantics used in the post-process would increase significantly and the inter- and intra-head discrepancy would go worse, they may lead to degradation or even failures in methods ignoring it. 
Also, we note our obvious advantage in the models with specially designed feature aggregation modules, like attention in LSTR, feature flip in BézierLaneNet, and spatial convolution in SCNN and RESA. Those modules usually require information aggregated from the total network, which leads the layer-wise and block-wise methods to a harder situation, while our network-wise framework could take advantage of the cross-layer relationship naturally. Even in hard cases like LaneATT, SCNN, and RESA, the method could outperform others significantly.
The proposed Selective Focus framework leverages the post-process information in the PTQ stage and thus tunes the quantized model more efficiently. 
With advanced performance across different precision configurations and model types, we achieve the new state-of-the-art post-training quantized lane detection and reduce the tuning time by more than 6x. 

\begin{table}[hbt]
	\centering
	\begin{adjustbox}{max width=0.45\textwidth}
		\setlength\tabcolsep{1.5pt}
		\begin{tabular}{ccc}
			\toprule
			Method              & Duration (Minutes) & F1 Score \\
			\midrule
			Ours                & 32                 & 76.31    \\
			w/o Focus           & 31                 & 73.27    \\
			w/o Selection       & 29                 & 75.30    \\
			w/o Focus+Selection & 29                 & 72.98    \\
			\bottomrule
		\end{tabular}
	\end{adjustbox}
	\caption{
		Abalation study of the proposed framework. Each component contributes to the proposed framework, and the two components are mutually beneficial. 
	}
	\label{tab:abl}
\end{table}

\begin{table}[hbt]
	\centering
	\begin{adjustbox}{max width=0.45\textwidth}
		\setlength\tabcolsep{1.5pt}
		\begin{tabular}{cccc}
			\toprule
			Network                            & Method                         & Duration (Minutes)             & F1 Score                        \\
			\midrule
			\multirow{3}{*}{CondLaneNet Small} & \cellcolor[HTML]{d0d0d0}{LSQ+} & \cellcolor[HTML]{d0d0d0}{1303} & \cellcolor[HTML]{d0d0d0}{76.92} \\
			                                   & QDrop                          & 112                            & 74.76                           \\
			                                   & Ours                           & 33                             & 75.56                           \\
			\midrule
			\multirow{3}{*}{RESA Small}        & \cellcolor[HTML]{d0d0d0}{LSQ+} & \cellcolor[HTML]{d0d0d0}{5326} & \cellcolor[HTML]{d0d0d0}{69.80} \\
			                                   & QDrop                          & 4378                           & 67.59                           \\
			                                   & Ours                           & 46                             & 69.46                           \\
			\bottomrule
		\end{tabular}
	\end{adjustbox}
	\caption{Comparison between PTQ and QAT methods. The QAT method LSQ+ (gray region) suffers from low computation efficiency, while conventional PTQ methods such as QDrop could save the cost but with an obvious performance drop. In contrast, the proposed method significantly improve the computation efficiency with less performance gap. 
	}
	\label{tab:qat}
\end{table}

\subsection{Ablation Study}
We first investigate the effect of each component of the proposed framework. Then, we analyze the efficiency of our method, compared to QAT and previous PTQ methods.

\subsubsection{Component Analysis}
To elucidate the contributions of individual components in our proposed method, we conducted an ablation study on GANet Small using the W4A4 quantization configuration, as detailed in Table~\ref{tab:abl}.  In comparison to the basic network-wise alignment (w/o Focus+Selection), our approach boosts the performance by over 3\% in the F1 score. The Semantics Guided Focus emerges as the primary performance driver, underscoring the significance of foreground information and the separate reconstruction benefits for semantics and confidence.  While the standalone Sensitivity Aware Selection module enhances the F1 score by a modest 0.3\%, its cooperation with Focus amplifies the improvement to 1\%, proving the framework's capability to manage semantic sensitivities both within and across heads. Notably, the proposed optimization strategy achieves performance on par with block-wise PTQ, yet maintains a speed akin to network-wise reconstruction.
Besides, experiments on the impact of the hyperparameters of Selection and the dynamic sensitivity property during optimization are all listed in Appendix~C.

\subsubsection{Efficiency Analysis}
  
Although QAT comes with cost and privacy concerns, it remains the premier quantization algorithm due to its promising performance. To evaluate the performance and efficiency gap between QAT and PTQ, we performed comparative experiments on CondLaneNet and RESA, utilizing the W8A4 quantization setup. These selected models differ in computational demands, allowing us to thoroughly probe the disparities between QAT and PTQ. Besides, though block-wise PTQ methods, such as QDrop, are much faster than QAT, the storage overhead and processing time for feature maps in intermediate layers are still problems. This is also the reason that we choose the efficient network-wise reconstruction, bringing a 6x speedup.

\section{Conclusion}

This paper sheds light on the post-training quantization in lane detection models leveraging the inherent semantics sensitivity. Our study delves into the essence of semantic sensitivity in the post-process and proposes a novel pipeline for identifying the sensitivity and further leveraging it for optimization.
By utilizing the post-processing information, the proposed framework boosts the performance of PTQ for lane detection even with the simplest optimization manner, which could motivate further exploration of the unused information lies in the lane detection models. Future endeavors might encompass efficient embedding of semantics information from post-processing—bypassing intermediate proxies more than the proposed score.

\section*{Acknowledgement}
This work was supported in part by the National Natural Science Foundation of China (No. 62206010, No.62022009), and the State Key Laboratory of Software Development Environment (SKLSDE-2022ZX-23).

\bibliography{aaai24}

\clearpage
\appendix

\section{Ilustations About Intra-head Sensitivity}\label{app:obs}

Many intra-head sensitivity illustrations were excluded from the main document due to space constraints. Figure~\ref{fig:appill} presents sensitivity-noise curves that provide insights into the sensitivity of post-processes to semantic errors. This figure allows for an intuitive comparison of different models under the defined Lane Distortion Score. Segmentation-based and anchor-based models exhibit greater robustness to noise. In contrast, curve-based and keypoint-based models show significant deviation in the presence of substantial noise.

\section{Methods}\label{app:meth}

\subsection{Lane Distortion Score Details}

For accurate matching within the Lane Distortion Score, lanes corresponding to the points must be identified and verified. The Intersection-over-Union (IoU) of the lanes is first computed to determine if they can be paired. Once lanes are matched, we proceed from the bottom to the top of the bounding box, considering points within a 1-pixel height error as matched. Points that remain unmatched in the lanes are termed mismatched points. It is imperative to understand that this score is specifically tailored for distorted lanes; thus, its applicability is limited to lanes with similar characteristics.

\subsection{Semantic Guided Focus Proof}

\begin{thm}
	Given $\mathcal{M}$ representing a matrix function that discerns elements associated with foreground or background regions, and $\mathcal{C}$ denoting the confidence function of FP models, offering confidence scores for semantics linked with the foreground, the following inequation stands:
	\begin{equation}
		\begin{aligned}
			       & \mathds{E}_{\mathbf{x}\sim \mathcal{D}}  \sum_{i} \left\|\mathcal{M}^i(\mathbf{x})\cdot(\hat{\mathcal{S}}^i(\mathbf{x}) - \mathcal{S}^i(\mathbf{x}))\right\|^2_F \\
			\leq ~ & \mathds{E}_{\mathbf{x}\sim \mathcal{D}}  \sum_{i}\left\|\mathcal{C}^i(\mathbf{x}) \odot(\hat{\mathcal{S}}_i(\mathbf{x}) - \mathcal{S}^i(\mathbf{x}))\right\|^2_F 
		\end{aligned}
	\end{equation}
\end{thm}

\begin{proof}
	Without loss of generality, we take the $i$-th term for simplicity and denote the symbols involved as vectors. \begin{equation}
	\begin{aligned}
		  & \mathds{E}_{\mathbf{x}\sim \mathcal{D}} \left\|\mathcal{M}^i(\mathbf{x})\cdot(\hat{\mathcal{S}}^i(\mathbf{x}) - \mathcal{S}^i(\mathbf{x}))\right\|^2_F                                     \\
		= & \mathds{E}_{\mathbf{x}\sim \mathcal{D}} \sum_j \mathcal{M}^i_j(\mathbf{x})\cdot(\hat{\mathcal{S}}^i_j(\mathbf{x}) - \mathcal{S}^i_j(\mathbf{x}))^2                                         \\
		= & \sum_j \mathds{E}_{\mathbf{x}\sim \mathcal{D}}\mathcal{M}^i_j(\mathbf{x})\cdot\mathds{E}_{\mathbf{x}\sim \mathcal{D}}(\hat{\mathcal{S}}^i_j(\mathbf{x}) - \mathcal{S}^i_j(\mathbf{x}))^2 + \\
		  & \mathrm{Cov}\left(\mathcal{M}_j^i(\mathbf{x}), \hat{\mathcal{S}}^i(\mathbf{x}) - \mathcal{S}^i(\mathbf{x}))^2\right).                                                                      
	\end{aligned}
	\end{equation} Similarly, the right of the equation could be expanded as \begin{equation}
	\begin{aligned}
		  & \mathds{E}_{\mathbf{x}\sim \mathcal{D}} \left\|\mathcal{C}^i(\mathbf{x})\odot(\hat{\mathcal{S}}^i(\mathbf{x}) - \mathcal{S}^i(\mathbf{x}))\right\|^2_F                                     \\
		= & \sum_j \mathds{E}_{\mathbf{x}\sim \mathcal{D}}\mathcal{C}^i_j(\mathbf{x})\cdot\mathds{E}_{\mathbf{x}\sim \mathcal{D}}(\hat{\mathcal{S}}^i_j(\mathbf{x}) - \mathcal{S}^i_j(\mathbf{x}))^2 + \\
		  & \mathrm{Cov}\left(\mathcal{C}_j^i(\mathbf{x}), \hat{\mathcal{S}}^i(\mathbf{x}) - \mathcal{S}^i(\mathbf{x}))^2\right).                                                                      
	\end{aligned}
	\end{equation}
	Then, the goal becomes to prove \begin{equation}\label{eq:lls}
	\begin{aligned}
		  & \mathrm{Cov}\left(\mathcal{M}_j^i(\mathbf{x}), \hat{\mathcal{S}}^i(\mathbf{x}) - \mathcal{S}^i(\mathbf{x}))^2\right)       \\
		  & \leq \mathrm{Cov}\left(\mathcal{C}_j^i(\mathbf{x}), \hat{\mathcal{S}}^i(\mathbf{x}) - \mathcal{S}^i(\mathbf{x}))^2\right). 
	\end{aligned}
	\end{equation} Since both confidence and semantics outputs originate from the same backbone, it is inappropriate to assume their correlation. In fact, this correlation should exceed that with the mask. This observation is valid for the diagonal entries of the covariance matrix. As for the non-diagonal entries of the covariance matrix, both sides are zero due to the independency across different samples. Therefore, Equation~\ref{eq:lls} is upheld. The theorem is proved. 
\end{proof}
From a sampling perspective, if the model generates an estimated mask $\hat{\mathcal{M}_j^i}$, sampled from the binomial distribution parameterized by the confidence output in each iteration, it must conform to: \begin{equation}
\mathds{E} \left(\Tilde{\mathcal{M}_j^i}\mid \mathbf{x}\right) = \mathcal{C}_j^i(\mathbf{x}).
\end{equation} 
Then the estimated expectation in the theorem could be computed with the total expectation law: 
\begin{equation}
	\begin{aligned}
		  & \mathds{E}_{\mathbf{x}\sim \mathcal{D}}  \sum_{i} \left\|\Tilde{\mathcal{M}}^i(\mathbf{x})\odot(\hat{\mathcal{S}}^i(\mathbf{x}) - \mathcal{S}^i(\mathbf{x}))\right\|^2_F                                     \\
		= & \mathds{E}_{\mathbf{x}\sim \mathcal{D}}  \sum_{i} \left\|\Tilde{\mathcal{M}}^i(\mathbf{x})\odot(\hat{\mathcal{S}}^i(\mathbf{x}) - \mathcal{S}^i(\mathbf{x}))\right\|^2_F                                     \\ 
		= & \mathds{E}_{\mathbf{x}\sim \mathcal{D}}  \sum_{i} \mathds{E}\left(\left\|\Tilde{\mathcal{M}}^i(\mathbf{x})\odot(\hat{\mathcal{S}}^i(\mathbf{x}) - \mathcal{S}^i(\mathbf{x}))\right\|^2_F | \mathbf{x}\right) \\ 
		= & \mathds{E}_{\mathbf{x}\sim \mathcal{D}}  \sum_{i} \left\|{\mathcal{C}}^i(\mathbf{x})\odot(\hat{\mathcal{S}}^i(\mathbf{x}) - \mathcal{S}^i(\mathbf{x}))\right\|^2_F,                                          
	\end{aligned}
\end{equation} which aligns with the intuitive notion that the model can predict a sufficiently accurate mask for the PTQ.

\begin{table*}[htbp]
	\begin{adjustbox}{max width=0.99\textwidth}
		\begin{tabular}{lcccccccccccccc}
			\toprule
			& \multicolumn{2}{c}{CondLaneNet} & \multicolumn{2}{c}{GANet} & \multicolumn{2}{c}{LSTR} & \multicolumn{2}{c}{BézierLaneNet} & \multicolumn{2}{c}{LaneATT} & \multicolumn{2}{c}{SCNN} & \multicolumn{2}{c}{RESA} \\
			           & Small & Mid   & Small & Mid   & Small & Mid   & Small & Mid   & Small & Mid        & Small & Mid   & Small & Mid   \\
			\midrule
			total\_tp  & 72486 & 73610 & 72662 & 75964 & 59547 & 68266 & 68752 & 69329 & 64173 & 32303      & 69651 & 70185 & 71701 & 72748 \\
			total\_fn  & 32400 & 31276 & 32224 & 28922 & 45339 & 36620 & 36134 & 35557 & 40713 & 72583      & 35235 & 34701 & 33185 & 32138 \\
			total\_fp  & 14489 & 16177 & 12367 & 17727 & 24198 & 27198 & 25785 & 23845 & 14676 & 49915      & 25582 & 26885 & 29878 & 28456 \\
			\midrule
			normal\_tp & 29084 & 29155 & 29116 & 29772 & 25468 & 27761 & 27939 & 28151 & 26974 & 14922      & 28687 & 28699 & 28935 & 29039 \\
			normal\_fn & 3693  & 3622  & 3661  & 3005  & 7309  & 5016  & 4838  & 4626  & 5803  & 17855      & 4090  & 4078  & 3842  & 3738  \\
			normal\_fp & 2472  & 2581  & 1907  & 2281  & 4172  & 3929  & 4060  & 3775  & 2609  & 14031      & 3383  & 3517  & 3606  & 3581  \\
			\midrule
			crowd\_tp  & 18997 & 19500 & 19162 & 19937 & 15927 & 17662 & 17946 & 17940 & 16529 & 7985       & 18111 & 18217 & 18670 & 18922 \\
			crowd\_fn  & 9006  & 8503  & 8841  & 8066  & 12076 & 10341 & 10057 & 10063 & 11474 & 20018      & 9892  & 9786  & 9333  & 9081  \\
			crowd\_fp  & 4701  & 4850  & 3569  & 4852  & 7175  & 8333  & 7582  & 7197  & 4631  & 14178      & 7395  & 7953  & 8382  & 7927  \\
			\midrule
			hlight\_tp & 953   & 993   & 974   & 1090  & 745   & 953   & 887   & 965   & 863   & 373        & 914   & 920   & 975   & 932   \\
			hlight\_fn & 732   & 692   & 711   & 595   & 940   & 732   & 798   & 720   & 822   & 1312       & 771   & 765   & 710   & 753   \\
			hlight\_fp & 362   & 341   & 227   & 321   & 491   & 500   & 539   & 469   & 296   & 783        & 495   & 555   & 556   & 578   \\
			\midrule
			shadow\_tp & 2028  & 2016  & 1972  & 2052  & 1347  & 1706  & 1666  & 1851  & 1514  & 669        & 1679  & 1764  & 1729  & 1882  \\
			shadow\_fn & 848   & 860   & 904   & 824   & 1529  & 1170  & 1210  & 1025  & 1362  & 2207       & 1197  & 1112  & 1147  & 994   \\
			shadow\_fp & 368   & 541   & 373   & 467   & 898   & 968   & 1009  & 828   & 554   & 1382       & 883   & 839   & 1088  & 923   \\
			\midrule
			noline\_tp & 5147  & 5299  & 5278  & 5990  & 3983  & 5084  & 5122  & 5063  & 4298  & 2115       & 4891  & 5168  & 5516  & 5788  \\
			noline\_fn & 8874  & 8722  & 8743  & 8031  & 10038 & 8937  & 8899  & 8958  & 9723  & 11906      & 9130  & 8853  & 8505  & 8233  \\
			noline\_fp & 2025  & 2698  & 1890  & 3465  & 3720  & 5381  & 5219  & 4685  & 2358  & 6101       & 5792  & 6228  & 6960  & 6567  \\
			\midrule
			arrow\_tp  & 2599  & 2604  & 2610  & 2697  & 2065  & 2427  & 2385  & 2457  & 2269  & 1217       & 2536  & 2541  & 2617  & 2645  \\
			arrow\_fn  & 583   & 578   & 572   & 485   & 1117  & 755   & 797   & 725   & 913   & 1965       & 646   & 641   & 565   & 537   \\
			arrow\_fp  & 271   & 249   & 198   & 284   & 515   & 492   & 545   & 424   & 304   & 1461       & 473   & 453   & 481   & 457   \\
			\midrule
			curve\_tp  & 771   & 788   & 835   & 884   & 576   & 684   & 625   & 663   & 630   & 312        & 736   & 721   & 754   & 777   \\
			curve\_fn  & 541   & 524   & 477   & 428   & 736   & 628   & 687   & 649   & 682   & 1000       & 576   & 591   & 558   & 535   \\
			curve\_fp  & 176   & 201   & 120   & 191   & 305   & 349   & 359   & 314   & 212   & 576        & 283   & 272   & 363   & 323   \\
			\midrule
			cross\_tp  & 0     & 0     & 0     & 0     & 0     & 0     & 0     & 0     & 0     & 0          & 0     & 0     & 0     & 0     \\
			cross\_fn  & 0     & 0     & 0     & 0     & 0     & 0     & 0     & 0     & 0     & 0          & 0     & 0     & 0     & 0     \\
			cross\_fp  & 1010  & 1136  & 1317  & 1881  & 1394  & 1417  & 1231  & 995   & 906   & 1936       & 1533  & 1593  & 2092  & 1764  \\
			\midrule
			night\_tp  & 12907 & 13255 & 12715 & 13542 & 9436  & 11989 & 12182 & 12239 & 11096 & 4710       & 12097 & 12155 & 12505 & 12763 \\
			night\_fn  & 8123  & 7775  & 8315  & 7488  & 11594 & 9041  & 8848  & 8791  & 9934  & 16320      & 8933  & 8875  & 8525  & 8267  \\
			night\_fp  & 3104  & 3580  & 2766  & 3985  & 5528  & 5829  & 5241  & 5158  & 2806  & 9467       & 5345  & 5475  & 6350  & 6336  \\
			\bottomrule
		\end{tabular}
	\end{adjustbox}
	\caption{Full metrics on the CULane dataset of our method under W4A4 quantization configuration. Further precision/recall/F1 scores could be derived from those data.}\label{tab:ful}
\end{table*}

\begin{table*}[htbp]
	\centering
	\begin{tabular}{llllllllll}
		\toprule
		Precision & Method & \multicolumn{2}{c}{BezierLaneNet} & \multicolumn{2}{c}{GANet} & \multicolumn{2}{c}{LSTR} & \multicolumn{2}{c}{SCNN} \\
		\midrule
		wa88                  & ACIQ     & 94.96          & 95.35          & 97.60          & 95.80          & 93.95          & 94.04          & 92.47          & 94.73          \\
		\midrule
		\multirow{6}{*}{w8a4} & ACIQ     & 49.99          & 66.54          & 39.29          & 72.36          & 90.59          & 92.70          & 60.85          & 82.88          \\
		                      & OMSE     & 91.95          & 92.73          & 93.61          & 91.79          & 92.02          & 93.05          & 85.17          & 91.99          \\
		                      & AdaRound & 92.85          & 93.73          & 92.08          & 90.77          & 94.05          & 94.41          & 80.19          & 89.44          \\
		                      & BRECQ    & 94.39          & 94.81          & 95.00          & 92.48          & 93.94          & 94.56          & 89.12          & 93.61          \\
		                      & QDrop    & 94.13          & 94.76          & 94.87          & 92.38          & 94.03          & 94.43          & 88.95          & 93.27          \\
		                      & Ours     & \textbf{94.88} & \textbf{95.27} & \textbf{96.60} & \textbf{94.78} & \textbf{94.75} & \textbf{94.63} & \textbf{92.31} & \textbf{94.66} \\
		\midrule
		\multirow{6}{*}{wa44} & ACIQ     & 48.13          & 65.30          & 19.07          & 58.17          & 4.85           & 45.15          & 48.94          & 77.85          \\
		                      & OMSE     & 87.89          & 89.88          & 88.71          & 90.75          & 4.59           & 44.67          & 77.36          & 88.69          \\
		                      & AdaRound & 94.72          & 94.91          & 0.00           & 0.00           & 46.19          & 66.98          & 14.08          & 63.37          \\
		                      & BRECQ    & 94.86          & 95.23          & 95.91          & 94.82          & 77.03          & 87.24          & 67.32          & 85.63          \\
		                      & QDrop    & 94.54          & 95.07          & 96.24          & 94.50          & 42.61          & 64.44          & 90.37          & 93.82          \\
		                      & Ours     & \textbf{95.04} & \textbf{95.33} & \textbf{96.55} & \textbf{94.68} & \textbf{86.22} & \textbf{90.10} & \textbf{91.98} & \textbf{94.70} \\
		\bottomrule
	\end{tabular}
	\caption{TuSimple performance comparison among different quantization algorithms and models, where the first column is the F1 score and the second is the accuracy. W8A8 means the weight and
		activation are all quantized into 8 bits, and so does W4A4 and W8A4. Note that the 8-bit quantization has almost no performance effect on the TuSimple dataset, so we skip the training-based methods on it.}\label{tab:tus}
\end{table*}

\subsection{Semantic and Architectural Discrepancies in Quantization}

The PTQ method, as detailed in \cite{ptq4vit}, targets discrepancies arising from architecture during the forward pass of quantized models. It identifies that the quantization of multi-head attention (MHA) modules can alter the relative order of attention maps, leading to performance deterioration. To address this, it introduces a ranking-aware loss that preserves the order across different attention heads.

Conversely, our research delves into the diverse functionalities in post-processing, highlighting task-specific variance. While prior post-training approaches concentrated on aligning activations between quantized models and their float-point counterparts, we recognize that the significance of these activations varies depending on the semantics.  To the best of our knowledge, this work is the first to leverage non-differential post-processing to quantize detection models, thereby boosting both performance and efficiency.

\cite{ptq4vit} and our work respectively focus on the intermediate outputs of MHA modules, and the network's final outputs in terms of semantic sensitivity. This opens up the possibility of integrating these two methods for enhanced overall performance, given their non-conflicting optimization paths.

\section{Experiments}\label{app:exp}
\subsection{Implementation Details}
The proposed quantization algorithms are structured into three distinct stages: (1) preparation, (2) calibration, and (3) tuning, as detailed below.

\subsubsection{Prepare}\label{appsec:prepare}
To initiate the process, a sensitivity-noise curve for each model is established, which subsequently aids in Sensitivity Aware Selection. We randomly sample 100 images from the training dataset to act as data points. Initially, the full-precision model predicts the lanes of these images, caching outputs from every head. For each head, varying noise levels are applied to these cached outputs. Subsequently, the distorted lanes are decoded from the post-process. The Lane Distortion Scores, calculated at every noise level and data point, act as the proxy function in Sensitivity Aware Selection after being interpolated.

\begin{figure}[htbp]
	\centering
	\includegraphics{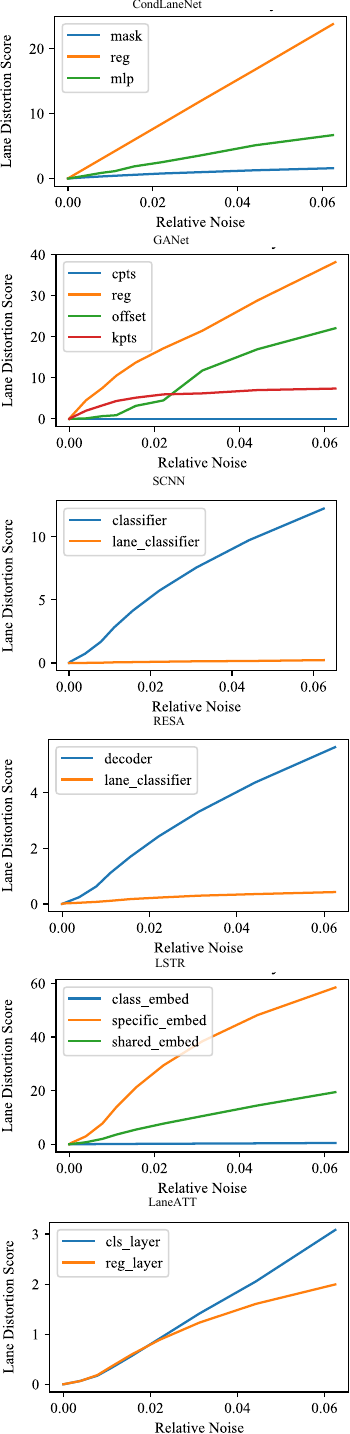}
	\caption{Illustration of inter-head semantic sensitivity of more lane detection models with more datapoints.}
	\label{fig:appill}
\end{figure}

\begin{table*}[h]
	\centering
	\begin{tabular} {cccccccccc}
		\toprule
		Network     & Task     & \# Head & Dataset    & Metric   & FP   & W/A & OMSE & QDrop & Ours \\
		\midrule
		CenterNet   & det 2d   & 3       & COCO       & bbox mAP & 25.9 & 4/4 & 5.8  & 11.0  & 13.3 \\
		CondlaneNet & lane det & 6       & CurveLanes & F1 Score & 85.1 & 4/4 & 76.9 & 81.4  & 82.8 \\
		\bottomrule
	\end{tabular}
	\caption{Abaltaion on other detection tasks and datasets.}
	\label{tab:general}
\end{table*}

\begin{table}[htpb]
	\centering
	\begin{tabular}{c|c|c}
		\toprule
		$k$ & Accuracy & F1 Score \\
		\midrule
		1   & 89.31    & 91.66    \\
		2   & 94.66    & 96.46    \\
		3   & 94.75    & 96.51    \\
		4   & 94.60    & 96.23    \\
		\bottomrule
	\end{tabular}
	\caption{Abaltaion of the hyperparameter effect in the Sensitivity Aware Selection on GANet~(Small) and TuSimple Dataset.}
	\label{tab:kontu}
\end{table}

The sensitivity curve construction needs 100 unlabeled images across 8 noise levels, with 20 reruns at each level, which means every image needs 1 forward and 160 post-processes. It totals approximately 10 minutes per network — a minor cost compared to the training part.
\subsubsection{Calibrate}
For calibration, the quantized models use 512 random images from the training dataset, determining the scale (with a zero-point consistently set at 0 for symmetric quantization). The augmentation follows the full-precision models without any additional modifications. While the OMSE calibration technique is employed, an exponential moving average MSE is also utilized for activations to enhance performance. Once the scales are ascertained, they are frozen during the tuning.
\subsubsection{Tuning}
After obtaining a calibrated quantized model, we implement training-based PTQ. Mini-batches are derived from the calibration data without any further data augmentations. The target outputs are the cached head outputs from the full-precision models, whereas the optimization objective is defined in Equation~(6) for the selected heads. The optimization process spans 5000 iterations, updating the selection after every 2000 iterations.
With block assignment~\cite{brecq, qdrop} and specific tuning strategies~\cite{adaround,adaquant,nwq}, our method sets a new benchmark using straightforward training configurations. This stellar performance is attributed to our well-designed methods for knowledge extraction from the post-process.

\subsection{Full Metrics On The CULane Dataset}
The evaluation on the CULane dataset incorporates metrics for various scenarios. We present the metrics of our W8A4 method in Table~\ref{tab:ful}. A notable observation is a significant discrepancy in false positives among models, indicating that performance primarily hinges on the model's ability to avoid misidentifying lanes under limited precision.

\subsection{Extended Experiments On The TuSimple Dataset}
The TuSimple dataset, comprising images from highways, consists of 3626 training images and 2782 testing images. The primary evaluation metric is accuracy. Given the straightforward nature of the dataset, most models exhibit no quantization issues at higher precisions such as 8-bit. Experiments were therefore conducted on W8A4 and W4A4 for comparison, with the results detailed in Table~\ref{tab:tus}. Notably, our model consistently outperforms others, boasting improvements exceeding 1.5% in accuracy and 2.0% in F1 score.

\subsection{Dynamic Sensitivity}

An ablation study on the dynamic selection's hyper-parameter is provided herein. With four heads in the GANet model, the top-$k$ selection was set to values 1 through 4. Table~\ref{tab:kontu} demonstrates that smaller $k$ values result in underoptimization, while larger ones can cause distraction. Since our intention is to focus on the most sensitive heads, this hyperparameter must be chosen judiciously.

We also documented the relative noise levels of heads, depicted in Figure~\ref{fig:dynaganet}. Initial observations indicate that the noise level in \texttt{reg} and \texttt{cpt} heads are considerably higher, leading to their early selection. This results in an overall reduction in noise levels, possibly due to a decrease in accumulative noise from the backbone and neck. As the research progresses, noise levels equilibrate, leading to the selection of sensitive heads. Although the noise levels in some heads subsequently increase, it has a negligible impact on the post-process, underscoring the efficacy of our selection process. Two key insights emerge: (1) shared components of the model influence noise distribution among heads, and (2) once shared components are stabilized, attention can be diverted to the sensitive heads.

\subsection{Generalization Abalation}

This proposed principle of leveraging semantics is also applicable to other tasks and other complex datasets, due to the common semantical variances across heads. We substantiate it with empirical results in Table~\ref{tab:general}. The normal detection tasks on MS COCO~\cite{mscoco} could be improved by 2.5\% bbox mAP, showing consistent improvement. The curved lane detection on CurveLanes~\cite{curvelane} could also be improved by 1.4\% F1 Score. 
Lane detection's positive region (lane) is much sparser than others (bbox, etc.), which makes the post-process more sensitive and semantics more useful. 

\subsection{Performance Analysis}

The proposed methods incur little overhead in preparation and achieve a large speed-up. 
Before optimization, the sensitivity-noise curves of each model should be built. As shown in Section~\ref{appsec:prepare}, we just need to forward each image once and repeat noised post-processing, which takes only 10 minutes in total.
In the optimization phase of post-training quantization, the proposed methods require 5k iterations in total, compared to 20k iterations per block needed by prior approaches. Semantic focus modeling is the key to efficiency and effectiveness, with a pre-processing in only 10 minutes for each model. Overall, our training is completed in under an hour, while previous methods require from two hours to days.

\begin{figure}[htbp]
	\centering
	\includegraphics{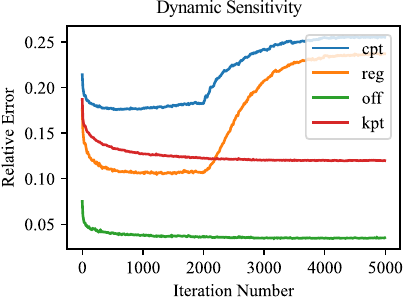}
	\caption{Dynamic noise level of GANet during its PTQ optimization with the proposed method. At the first 2k iterations, \texttt{reg} and \texttt{cpt} are chosen, and the \texttt{kpt} and \texttt{off} are chosen at the left time.}
	\label{fig:dynaganet}
\end{figure}

\end{document}